\documentclass[letterpaper]{article}

\usepackage{kr}

\usepackage{times}
\usepackage{soul}
\usepackage{url}
\usepackage[hidelinks]{hyperref}
\usepackage[utf8]{inputenc}
\usepackage[small]{caption}
\usepackage{graphicx}
\usepackage{amsmath}
\usepackage{amsthm}
\usepackage{booktabs}
\usepackage{algorithm}
\usepackage{enumitem}
\usepackage{amssymb}
\usepackage{algorithm}
\usepackage{algpseudocode}

\newtheorem{theorem}{Theorem}
\newtheorem{prop}{Proposition}
\theoremstyle{definition}
\newtheorem{definition}{Definition}
\newtheorem{example}{Example}
\usepackage{wasysym}

\usepackage{tikz}
\usepackage{todonotes}
\usepackage{subcaption}
\usepackage{multirow}

\title{
Hierarchical Reinforcement Learning for Deep Goal Reasoning:\\ An Expressiveness Analysis
}

\author{%
Weihang Yuan\and
Héctor Muñoz-Avila\\
\affiliations
Department of Computer Science and Engineering\\
Lehigh University\\
\emails
\{wey218, hem4\}@lehigh.edu
}

\begin{document}

\maketitle

\begin{abstract}

Hierarchical DQN (h-DQN) is a two-level architecture of feedforward neural networks where the meta level selects goals and the lower level takes actions to achieve the goals. We show tasks that cannot be solved by h-DQN, exemplifying the limitation of this type of hierarchical framework (HF). We describe the recurrent hierarchical framework (RHF), generalizing architectures that use a recurrent neural network at the meta level. We analyze the expressiveness of HF and RHF using context-sensitive grammars. We show that RHF is more expressive than HF. We perform experiments comparing an implementation of RHF with two HF baselines; the results corroborate our theoretical findings.

\end{abstract}

\section{Introduction}


Hierarchical DQN (h-DQN) \cite{kulkarni2016hierarchical} is a reinforcement learning (RL) system that combines deep RL and goal reasoning. h-DQN uses a two-level hierarchical architecture: a meta controller at the top level and a controller at the lower level. The meta controller and the controller operate at different temporal scales.  The meta controller receives a state $s$ and selects a goal $g$, which can be achieved in some state. The controller receives the current state and $g$; it takes actions until $g$ is achieved in a state $s'$. Then the meta controller receives $s'$ and selects the next goal $g'$; the process repeats. Both levels of h-DQN use deep Q-networks (DQN) \cite{mnih2015human} to jointly learn value functions; the meta controller tries to maximize the cumulative reward from the environment whereas the controller is motivated by user-defined intrinsic rewards. 


Goal reasoning is the study of agents that are capable of reasoning about their goals and changing them when the need arises \cite{aha2018goal,munoz2019everything}. h-DQN selects its goals and plans toward achieving them while adapting the goal selection process using RL. Thus h-DQN can be considered a goal reasoning agent.

To exemplify the capability and limitation of h-DQN, consider a one-dimensional corridor environment shown in Figure~\ref{fig:corridor} (top), which is based on the environments in \cite{osband2016deep,kulkarni2016hierarchical}. The corridor consists of 7 states from left to right: $s_0, \dots, s_6$. An agent starts in $s_3$ (circled) and can move left or right. $s_0$ is the terminal state (gray). In task (i), the agent receives a reward of $+1$ if it visits $s_6$ (asterisk) at least once and then reaches $s_0$; otherwise, the reward is $+0.01$. Figure~\ref{fig:corridor}~(a) shows a trajectory that results in a reward of $+1$ in task (i). h-DQN is able to generate this trajectory and obtain the maximum reward in task (i). Now consider task (ii): the agent must visit $s_6$ at least twice to receive a reward of $+1$. We will show that it is impossible for h-DQN to deterministically generate a trajectory (e.g., Figure~\ref{fig:corridor}~(b)) that solves task (ii).

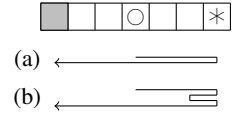
\begin{figure}
    \centering
    \begin{tikzpicture}[scale=0.36]
    \foreach \i in {0, ..., 7} {
        \draw (\i, 0) -- (\i, 1);
    }
    \foreach \i in {0, 1} {
        \draw  (0, \i) -- (7, \i);
    }
    \draw [fill=lightgray] (0, 0) rectangle (1, 1);
    \draw node at (3.5,0.5) {\fontsize{9}{9}\Circle};
    \draw node at (6.5,0.5) {\fontsize{9}{9}\varhexstar};
    \draw  (3.5, -1) -- (6.5, -1);
    \draw  (6.5, -1) -- (6.5, -1.2);
    \draw[->]  (6.5, -1.2) -- (0.5, -1.2);
    \draw node at (-0.5, -1.1) {\small(a)};
    \draw  (3.5, -2.2) -- (6.5, -2.2);
    \draw  (6.5, -2.2) -- (6.5, -2.4);
    \draw  (6.5, -2.4) -- (5.5, -2.4);
    \draw  (5.5, -2.4) -- (5.5, -2.6);
    \draw  (5.5, -2.6) -- (6.5, -2.6);
    \draw  (6.5, -2.6) -- (6.5, -2.8);
    \draw[->]  (6.5, -2.8) -- (0.5, -2.8);
    \draw node at (-0.5, -2.5) {\small(b)};
\end{tikzpicture}
    \caption{The corridor environment (top). The starting state is circled. The terminal state is gray. In corridor task (ii), trajectory (a) results in a reward of $+0.01$; trajectory (b) results in a reward of $+1$.}
    \label{fig:corridor}
\end{figure}

h-DQN can be considered an instance of the hierarchical framework (HF): a two-level architecture of feedforward neural networks that can use any RL algorithms, not just Q-learning. HF can use different algorithms for the meta controller and the controller, whereas h-DQN uses the same algorithm at each level.

Consider using a recurrent neural network (RNN) in place of the feedforward network in the HF meta controller. We refer to this type of architecture as the recurrent hierarchical framework (RHF). A related instance is feudal networks (FuNs) \cite{vezhnevets2017feudal}, which uses RNNs at both levels. To investigate the effects of a recurrent meta controller, we focus on the case where HF and RHF use the same controller, isolating the difference in their meta controllers.

In the remainder of this paper, we describe HF and RHF, and construct two types of context-sensitive grammars that capture the expressiveness of HF and RHF respectively. Using the grammars, we show that RHF is more expressive than HF: (1) any state-goal trajectories generated by an HF system can be generated by an RHF system; (2) some state-goal trajectories can only be generated by RHF but not HF. We present an implementation of RHF. We perform experiments comparing this implementation with two HF baselines; the results corroborate our expressiveness analysis of HF and RHF.

\section{Related Work}

The options framework \cite{sutton1999between} defines an option as a temporally extended course of actions consisting of three components: an initiation set $I$, a policy $\pi$, and a termination condition $\beta$. An agent in a state $s$ can select an option $o$ if and only if $s\in I$. If $o$ is selected, the agent takes actions according to $\pi$ until $o$ terminates according to $\beta$.

DQN \cite{mnih2015human} combines a semi-gradient Q-learning algorithm \cite{watkins1992q} and a deep convolutional neural network (CNN) \cite{lecun1998gradient} to learn to play Atari 2600 games. To improve stability, DQN uses experience replay  and maintains both a Q-network and a target network. The Q-network provides the behavior policy and updates online. The target network, which is used to compute the temporal difference error, updates periodically and remains fixed between updates. DQN plays at human level or above in 29 of the games, but performs poorly in Montezuma's Revenge. This game requires strategic exploration due to sparse rewards. Hierarchical deep RL systems such as h-DQN \cite{kulkarni2016hierarchical} and FuNs \cite{vezhnevets2017feudal} achieve better performance than DQN in Montezuma's Revenge.

The deep recurrent Q-network (DRQN) \cite{hausknecht2015deep} replaces the fully-connected layer in DQN with a layer of long short-term memory (LSTM) \cite{hochreiter1997long}, which provides the capability to integrate state information through time. DRQN learns to play several partially observable versions of Atari games and achieves results comparable to DQN.

FuNs \cite{vezhnevets2017feudal} is a two-level architecture consisting of a manager and a worker. Both modules are recurrent: the worker uses a standard LSTM; the manager uses a dilated LSTM that extends its state memory. The manager is similar to a meta controller. The distinction is that the goals are directional in FuNs whereas the goals are states in h-DQN.

\section{Background}

\paragraph{Reinforcement Learning.} We consider reinforcement learning for episodic tasks. A decision-making agent interacts with an external environment. The environment is modeled by a finite set of nonterminal states $\mathcal{S}$, a terminal state $\tau$, a finite set of actions $\mathcal{A}$, a transition function $\mathcal{T}:\mathcal{S}\times\mathcal{A} \mapsto \mathcal{S}$, and a reward function $\mathcal{R}:\mathcal{S}^+\mapsto \mathbb{R} $. ($\mathcal{S}^+$ is the Kleene plus operation on in $\mathcal{S}$.) At each time step $t$, the agent observes a state $s_{(t)}\in\mathcal{S}$ and takes an action $a_{(t)}\in\mathcal{A}$. At the next time step, the environment transitions to $s_{(t+1)}$;\footnote{The parentheses in a subscript indicate that the subscript denotes a time step. For example, $s_{(0)}$ is the environment's state at time step 0. Without the parentheses, the subscript denotes a distinct element in a set. This notation applies to states, actions, and goals.} the agent receives a reward $r_{t+1}$. The process repeats until the terminal state is reached. The objective of the agent is to take actions in a way that maximizes the discounted return \cite{sutton2018reinforcement}:
    \begin{equation*}
        G_t = \sum_{i=0}^T\gamma^{i}r_{t+i+1},
    \end{equation*}
where $\gamma$ is the discount rate, $0\leq\gamma\leq1$, and $T$ is the terminal time step.

\paragraph{Context-Sensitive Grammars.} A context sensitive grammar (CSG) is a 4-tuple $(N,\Sigma,P,S)$, where $N$ is a finite set of nonterminal symbols, $\Sigma$ is a finite set of terminal symbols, $S$ is the start symbol, and $P$ is a finite set of production rules. All the rules in $P$ are of the form
\begin{equation*}
\alpha A \beta \to \alpha \gamma \beta,
\end{equation*}
where $A\in N$, $\alpha, \beta \in (\Sigma \cup N)^*$, and $\gamma \in (\Sigma \cup N)^+$.

Given a set of strings $V$, the $i$-th power of $V$ (i.e., the concatenation of $V$ with itself $i$ times) is recursively defined as follows:
\begin{align*}
    V^0 & = \{\lambda\}\\
    V^{i+1} & = \{wv\mid w\in V^i \text{ and } v\in V \}
\end{align*}
where $\lambda$ is the empty string.

\section{Expressiveness Analysis}

When an HF or RHF system interacts with an environment, the meta controller receives states and generates goals, forming a trajectory of alternating states and goals. For example, the string $s_3g_6s_6g_0s_0$ represents the state-goal trajectory in Figure~\ref{fig:corridor}~(a). To analyze the expressiveness of HF and RHF, we define two special types of CSGs that generate strings representing state-goal trajectories. A synopsis of our arguments is as follows:
\begin{enumerate}
    \item We describe how HF and RHF generate goals in an environment. Then we define two types of CSGs: \textit{constrained} and \textit{$k$-recurrent}, to represent HF's and RHF's goal generation respectively.

    \item We prove that a constrained CSG is a special case of a $k$-recurrent CSG. In other words, for any constrained CSG, there exists a $k$-recurrent CSG such that both grammars generate the same strings.
    \item We prove that constrained CSGs cannot generate some string. However, there exists a $k$-recurrent CSG that can generate this string.
\end{enumerate}
Combined, these points imply that RHF is more expressive than HF.

In our analysis, we make the assumption of deterministic policies: the system selects goals and actions deterministically after training. RL algorithms oftentimes use an exploratory stochastic policy. For example, h-DQN uses an $\epsilon$-greedy policy to select goals and actions during training; the best goal or action is selected with a probability of $1-\epsilon$ and a random one is selected with a probability of $\epsilon$. As learning proceeds, $\epsilon$ is gradually reduced to a small value. We want to analyze the system's behavior after it has converged to deterministic policies and hence make this assumption.

\subsection{Expressiveness of HF}

HF is a two-level architecture consisting of a meta controller and a controller. At a time step $t$, the meta controller receives a state $s_{(t)}\in\mathcal{S}$ and selects a goal $g_{(t)}\in\mathcal{G}$, where $\mathcal{G}$ is a finite set of goals. From $t$ to a time step $t+n$ when either $g_{(t)}$ is achieved or the terminal state is reached, the controller receives $s_{(t+i)}$ and $g_{(t)}$, and takes an action $a_{(t+i)}$, where $0\leq i \leq n-1$. At $t+n$, if the terminal state is not reached, the meta controller receives $s_{(t+n)}$ and selects the next goal $g_{(t+n)}$; the process repeats. Regardless of implementation details, under the assumption of deterministic policies, the meta controller defines a function $\mathcal{S} \mapsto \mathcal{G}$.

To capture the expressiveness of HF, we construct a type of CSG encompassing the following components:

\begin{enumerate}[label=(\alph*)]
     \item There can be one or more starting states. For every starting state $s$, we add one rule of the form: $S \to s\langle \text{META}\rangle$. $S$ is the start symbol. $\langle \text{META}\rangle$ is a nonterminal symbol representing the meta controller.

     \item For every state-goal pair $\lbrace s \rbrace \mapsto \lbrace g \rbrace$ defined by the function $\mathcal{S} \mapsto \mathcal{G}$, we add one rule of the form: $s\langle \text{META}\rangle \to sg\langle \text{ACT}\rangle s$. $\langle \text{ACT}\rangle$ is a nonterminal symbol representing the controller. The expression $sg\langle \text{ACT}\rangle$ is used to derive a new state, which is explained in (c).

    \item From $t$ to $t+n$, the controller receives a state-goal $(s_{(t+i)},g_{(t)})$ and takes an action $a_{(t+i)}$, where $0\leq i \leq n-1$. The iteration has three possible outcomes:
     \begin{enumerate}[label=\roman*.]
         \item The controller returns control to the meta controller when $g_{(t)}$ is achieved in $s_{(t+n)}$ and $s_{(t+n)}$ is not the terminal state. This is captured by the rule $sg\langle \text{ACT}\rangle \to sgs'\langle \text{META}\rangle$, where $s$ is the state when the meta controller gives control to the controller (i.e., $s = s_{(t)}$), $g$ is the goal selected by the meta controller (i.e., $g= g_{(t)}$), and $s'$ is the state satisfying $g$ (i.e., $s' = s_{(t+n)}$).
        
         \item The episode terminates when $s_{(t+n)}$ is the terminal state. It is possible that $g$ is achieved in $s_{(t+n)}$. This is captured by the rule $sg\langle \text{ACT}\rangle \to sg\tau$, where $\tau$ is the terminal state (i.e., $\tau = s_{(t+n)}$).

         \item The iteration gets stuck in an infinite loop because the controller never achieves $g$ nor the terminal state is ever reached. This is captured by the rule $sg\langle \text{ACT}\rangle \to sg\langle \text{ACT}\rangle$, which does not yield any string.
     \end{enumerate}

 \end{enumerate}

We formally define this type of CSG as follows.

\begin{definition}
A context-sensitive grammar $(N,\Sigma,P,S)$ is \textit{constrained} if
\begin{enumerate}
    \item The set of nonterminals $N=\{S, \langle \text{META}\rangle, \langle \text{ACT}\rangle\}$.
    \item The set of terminals $\Sigma=\mathcal{S}\cup\mathcal{G}\cup\{\tau\}$. 
    \item The set $P$ contains only production rules of the following forms:
\begin{enumerate}[label=(\alph*)]
    \item For every starting state $s\in\mathcal{S}$, there exists exactly one rule of the form
    \begin{equation*}
      S \to s\langle \text{META}\rangle.
    \end{equation*}
   
    \item For every $s\in\mathcal{S}$, there exists exactly one rule of the form
    \begin{equation*}
         s\langle \text{META}\rangle \to sg\langle \text{ACT}\rangle s,
    \end{equation*}

where $g\in\mathcal{G}$.

    \item For every combination of $s\in\mathcal{S}$ and $g\in\mathcal{G}$, there exists  exactly one rule of one of the following forms:
    
    \begin{enumerate}[label=\roman*.]
        \item $sg\langle \text{ACT}\rangle \to sgs'\langle \text{META}\rangle$, where $s'\in\mathcal{S}$.
        \item $sg\langle \text{ACT}\rangle \to sg\tau$.
        \item $sg\langle \text{ACT}\rangle \to sg\langle \text{ACT}\rangle$.
    \end{enumerate}

\end{enumerate}
\end{enumerate}
\end{definition}


\begin{example}
This example shows a constrained CSG that generates a string representing the state-goal trajectory in Figure \ref{fig:corridor}~(a), which explains how an HF system solves corridor task (i). We define that $g_i$ is achieved in $s_i$ ($0\leq i \leq 6$).

Consider a constrained CSG $G_1=(N,\Sigma,P,S)$. $\Sigma=\{s_i, g_i \mid 0\leq i \leq 6\}$. $s_0$ is the terminal state. Some rules in $P$ are:
\setcounter{equation}{0}
\begin{align}
    S &\to s_3\langle \text{META}\rangle\\
    s_3 \langle \text{META}\rangle &\to s_3g_6 \langle \text{ACT}\rangle s_3\\
    s_6 \langle \text{META}\rangle &\to s_6g_0 \langle \text{ACT}\rangle s_6\\
    s_3g_6 \langle \text{ACT}\rangle &\to s_3g_6s_6  \langle \text{META}\rangle  \\
    s_6g_0 \langle \text{ACT}\rangle &\to s_6g_0s_0
\end{align}
Applying these rules derives
\begin{equation*}
    \begin{split}
        S &\xrightarrow{1} s_3 \langle \text{META}\rangle\\
         &\xrightarrow{2} s_3 g_6\langle \text{ACT}\rangle s_3\\
         &\xrightarrow{4} s_3 g_6 s_6  \langle \text{META}\rangle  s_3\\
         &\xrightarrow{3}  s_3g_6 s_6g_0 \langle \text{ACT}\rangle s_6 s_3\\
         &\xrightarrow{5} s_3g_6 s_6g_0 s_0 s_6 s_3 \\
    \end{split}
\end{equation*}
The substring $s_3g_6 s_6g_0 s_0$ represents the target state-goal trajectory. The substring $s_0 s_6 s_3$ is the states visited in reverse order; it is an artifact of the derivation.
\end{example}

\begin{theorem}
Let $G=(N,\Sigma,P,S)$ be a constrained context-sensitive grammar and $\{s_0,g_a,g_b\}\subseteq\Sigma$. Then $G$ cannot generate a string that contains both $s_0g_a$ and $s_0g_b$ as substrings.
\end{theorem}

\begin{proof}
Observe that the only rules that add a $g$ to the right of an $s$ are of form (b) in Definition 1. Assume that $G$ can generate a string that contains both $s_0g_a$ and $s_0g_b$ as substrings. Then $P$ must contain at least the following rules:
\setcounter{equation}{0}
\begin{align}
    s_0 \langle \text{META}\rangle &\to s_0g_a \langle \text{ACT}\rangle s_0\\
    s_0 \langle \text{META}\rangle &\to s_0g_b \langle \text{ACT}\rangle s_0
\end{align}

Since $\{g_a, g_b\}\subseteq\Sigma$, the inequality $g_a \neq g_b$ holds. Since $P$ contains both rules (1) and (2), and there is at most one rule of the form $s_0\langle \text{META}\rangle \to s_0g\langle \text{ACT}\rangle s_0$ in $P$, the equality $g_a = g_b$ must hold, which contradicts the fact that $g_a \neq g_b$. Therefore Theorem 1 holds.

\end{proof}

Theorem 1 implies that constrained CSGs cannot generate a string that contains $s_3g_6s_6g_5s_5g_6s_6g_0s_0$ as a substring. This substring represents the state-goal trajectory in Figure \ref{fig:corridor} (b). This implies that HF systems cannot generate this trajectory.


\subsection{Expressiveness of RHF}
RHF is a two-level architecture consisting of a meta controller and a controller. RHF differs from HF in that the RHF meta controller receives as input a sequence of states instead a single state. Observe that the meta controller has a finite memory: the meta controller can recall a sequence of states of up to a certain length $k$. We define 
\begin{equation*}
\mathcal{S}^{\leq k}= \bigcup_{0\leq i\leq k}\mathcal{S}^i,
\end{equation*}
where $\mathcal{S}$ is the set of nonterminal states and $k$ is a nonnegative integer. $\tilde{s}\in\mathcal{S}^{\leq k}$ represents a chronologically ordered sequence of states observed by the meta controller before the current time step. At a time step $t$, the meta controller receives $s_{(t)}\tilde{s}$ and selects a goal $g\in\mathcal{G}$; then $s_{(t)}$ is appended to $\tilde{s}$.  From $t$ to a time step $t+n$ when either $g_{(t)}$ is achieved or the terminal state is reached, the controller receives $s_{(t+i)}$ and $g_{(t)}$, and takes an action $a_{(t+i)}$, where $0\leq i \leq n-1$. At $t+n$, if the terminal state is not reached, the meta controller receives $s_{(t+n)}\tilde{s}$ and selects the next goal $g_{(t+n)}$; then $s_{(t+n)}$ is appended to $\tilde{s}$; the process repeats. Therefore, regardless of implementation details, under the assumption of deterministic policies, the meta controller defines a function $\mathcal{S}^{\leq k+1} \mapsto \mathcal{G}$.

We construct a type of CSG as follows to capture the expressiveness of RHF and provide a formal definition afterward.

 \begin{enumerate}[label=(\alph*)]
     \item There can be one or more starting states. For every starting state $s$, we add one rule of the form: $S \to s\langle \text{META}\rangle$. $S$ is the start symbol. $\langle \text{META}\rangle$ is a nonterminal symbol representing the meta controller.

    \item For every sequence-goal pair $\{ s\tilde{s} \} \mapsto \{ g \}$ defined by the function $\mathcal{S}^{\leq k+1} \mapsto \mathcal{G}$, we add one rule of the form: $s\langle \text{META}\rangle\tilde{s} \to sg\langle \text{ACT}\rangle s\tilde{s}$. $\langle \text{ACT}\rangle$ is a nonterminal symbol representing the controller. The expression $sg\langle \text{ACT}\rangle $ is used to derive a new state, which is explained in (c).

     \item The controller receives a state-goal $(s,g)$ and behaves the same way as in HF. Three outcomes are possible:
     \begin{enumerate}[label=\roman*.]
         \item The controller returns control to the meta controller when $g$ is achieved in a state $s'$ and $s'$ is not the terminal state. This is captured by the rule $sg\langle \text{ACT}\rangle \to sgs'\langle \text{META}\rangle$.
        
        \item The episode terminates when the terminal state $\tau$ is reached. It is possible that $g$ is achieved in $\tau$. This is captured by the rule $sg\langle \text{ACT}\rangle \to sg\tau$.

        \item The iteration gets stuck in an infinite loop because the controller never achieves $g$ nor the terminal state is ever reached. This is captured by the rule $sg\langle \text{ACT}\rangle \to sg\langle \text{ACT}\rangle$.
     \end{enumerate}

 \end{enumerate}

\begin{definition}
A context-sensitive grammar $(N,\Sigma,P,S)$ is \textit{$k$-recurrent} ($k$ is a nonnegative integer) if
\begin{enumerate}
    \item  The set of nonterminals $N=\{S, \langle \text{META}\rangle, \langle \text{ACT}\rangle\}$.
    \item The set of terminals $\Sigma=\mathcal{S}\cup\mathcal{G}\cup\{\tau\}$.
    \item The set $P$ contains only production rules of the following forms:
\begin{enumerate}[label=(\alph*)]
    \item  For every starting state $s\in\mathcal{S}$, there exists exactly one rule of the form
    \begin{equation*}
     S\to s\langle \text{META}\rangle.
\end{equation*}
    \item For every combination of $s\in\mathcal{S}$ and $\tilde{s}\in\mathcal{S}^{\leq k}$, there exists zero or exactly one rule of the form
\begin{equation*}
     s\langle \text{META}\rangle\tilde{s} \to sg\langle \text{ACT}\rangle s\tilde{s},
\end{equation*}
where $g\in\mathcal{G}$.

   \item For every combination of $s\in\mathcal{S}$ and $g\in\mathcal{G}$, there exists  exactly one rule of one of the following forms:
    
    \begin{enumerate}[label=\roman*.]
        \item $sg\langle \text{ACT}\rangle \to sgs'\langle \text{META}\rangle$, where $s'\in\mathcal{S}$.
        \item $sg\langle \text{ACT}\rangle \to sg\tau$.
        \item $sg\langle \text{ACT}\rangle \to sg\langle \text{ACT}\rangle$.
    \end{enumerate}
\end{enumerate}
\end{enumerate}
\end{definition}


\begin{example}
This example shows a $k$-recurrent CSG that generates a string representing the state-goal trajectory in Figure~\ref{fig:corridor}~(b), which explains how an RHF system solves corridor task (ii). We define that $g_i$ is achieved in $s_i$ ($0\leq i \leq 6$).

Consider a 2-recurrent CSG $G_2=(N,\Sigma,P,S)$. $\Sigma=\{s_i, g_i \mid 0\leq i \leq 6\}$. $s_0$ is the terminal state. Some rules in $P$ are:
\setcounter{equation}{0}
\begin{align}
    S &\to s_3\langle \text{META}\rangle\\
    s_3 \langle \text{META}\rangle &\to s_3g_6 \langle \text{ACT}\rangle s_3\\
    s_6 \langle \text{META}\rangle s_3 &\to s_6g_5 \langle \text{ACT}\rangle s_6s_3\\
    s_5  \langle \text{META}\rangle s_6s_3 &\to s_5g_6 \langle \text{ACT}\rangle s_5 s_6s_3 \\
    s_6  \langle \text{META}\rangle s_5s_6s_3 &\to s_6s_0 \langle \text{ACT}\rangle s_6 s_5 s_6s_3 \\
    s_6g_5 \langle \text{ACT}\rangle &\to s_6g_5s_5 \langle \text{META}\rangle  \\
    s_3g_6 \langle \text{ACT}\rangle &\to s_3g_6s_6 \langle \text{META}\rangle \\
    s_5g_6 \langle \text{ACT}\rangle &\to s_5g_6s_6  \langle \text{META}\rangle \\
    s_6g_0 \langle \text{ACT}\rangle & \to s_6 g_0 s_0
\end{align}
Applying these rules derives
\begin{equation*}
    \begin{split}
        S &\xrightarrow{1} s_3\langle \text{META}\rangle \\
         &\xrightarrow{2} s_3g_6 \langle \text{ACT}\rangle s_3\\
         &\xrightarrow{7} s_3g_6s_6 \langle \text{META}\rangle s_3 \\
         &\xrightarrow{3} s_3g_6s_6g_5 \langle \text{ACT}\rangle s_6s_3\\
         &\xrightarrow{6}  s_3g_6 s_6g_5s_5 \langle \text{META}\rangle  s_6s_3  \\
        &\xrightarrow{4} s_3g_6 s_6g_5s_5g_6 \langle \text{ACT}\rangle s_5 s_6s_3 \\
       & \xrightarrow{8}  s_3g_6 s_6g_5s_5g_6s_6  \langle \text{META}\rangle s_5 s_6s_3\\
       & \xrightarrow{5} s_3g_6 s_6g_5s_5g_6s_6g_0 \langle \text{ACT}\rangle s_6 s_5 s_6s_3\\
        & \xrightarrow{9}s_3g_6 s_6g_5s_5g_6s_6g_0 s_0 s_6 s_5 s_6s_3
    \end{split}
\end{equation*}
The substring $s_3g_6s_6g_5s_5g_6s_6g_0s_0$ represents the target state-goal trajectory. The substring $s_0 s_6 s_5 s_6s_3$ is the states visited in reverse order; it is an artifact of the derivation.

\end{example}

\begin{prop}
Any constrained context-sensitive grammar is 0-recurrent.
\end{prop}

\begin{proof}
All constrained CSGs satisfy the conditions of Definition 2. Definition 1 is a special case of Definition 2 where $k=0$. Therefore Proposition 1 holds.

\end{proof}

Constrained CSGs and $k$-recurrent CSGs capture the expressiveness of HF and RHF respectively. Proposition 1 implies that for any constrained CSG, there exists a $k$-recurrent CSG such that both grammars generate the same strings. Theorem 1 implies that constrained CSGs cannot generate the string $s_3g_6 s_6g_5s_5g_6s_6g_0 s_0 s_6 s_5 s_6s_3$, which contains both $s_6g_5$ and $s_6g_0$. Example 2 shows that there exists a $k$-recurrent CSG that can generate this string. Therefore RHF is more expressive than HF.








\section{Implementation}

We describe an implementation of RHF named Rh\nobreakdash-REINFORCE, which is used in our experiments. Figure~\ref{fig:arc} shows the architecture of Rh-REINFORCE.

The meta controller is a policy approximator that uses an RNN. The input to the meta controller is a variable-length sequence of vectors, each of which represents a state observed by the meta controller at a time step. When the vector dimension is high (e.g., image pixels), the input is passed to two convolutional layers with rectifiers. The output of the convolutional layers is passed to a layer of  gated recurrent units (GRU) \cite{cho2014learning}. The output of the GRU layer is passed to a fully-connected softmax output layer  that has a number of units equal to $\lvert \mathcal{G} \rvert$. When the vector dimension is low (e.g., the corridor task), the input is directly passed to a GRU layer followed by an output layer. The last vector in the output sequence, which corresponds to the most recent time step, is used as a probability distribution to select a goal.

The controller is an actor-critic model \cite{konda2000actor} that uses two feedforward neural networks. The actor has two fully-connected or convolutional layers (depending on the input dimension) with rectifiers, followed by a fully-connected softmax output layer. The critic has the same network structure as the actor except that the output layer uses a linear activation instead of softmax. The input to the controller is a state extended by a vector representing the goal selected by the meta controller.

\begin{figure*}
\centering
\begin{tikzpicture}
\node at (0,0) {\includegraphics[width=.5\textwidth]{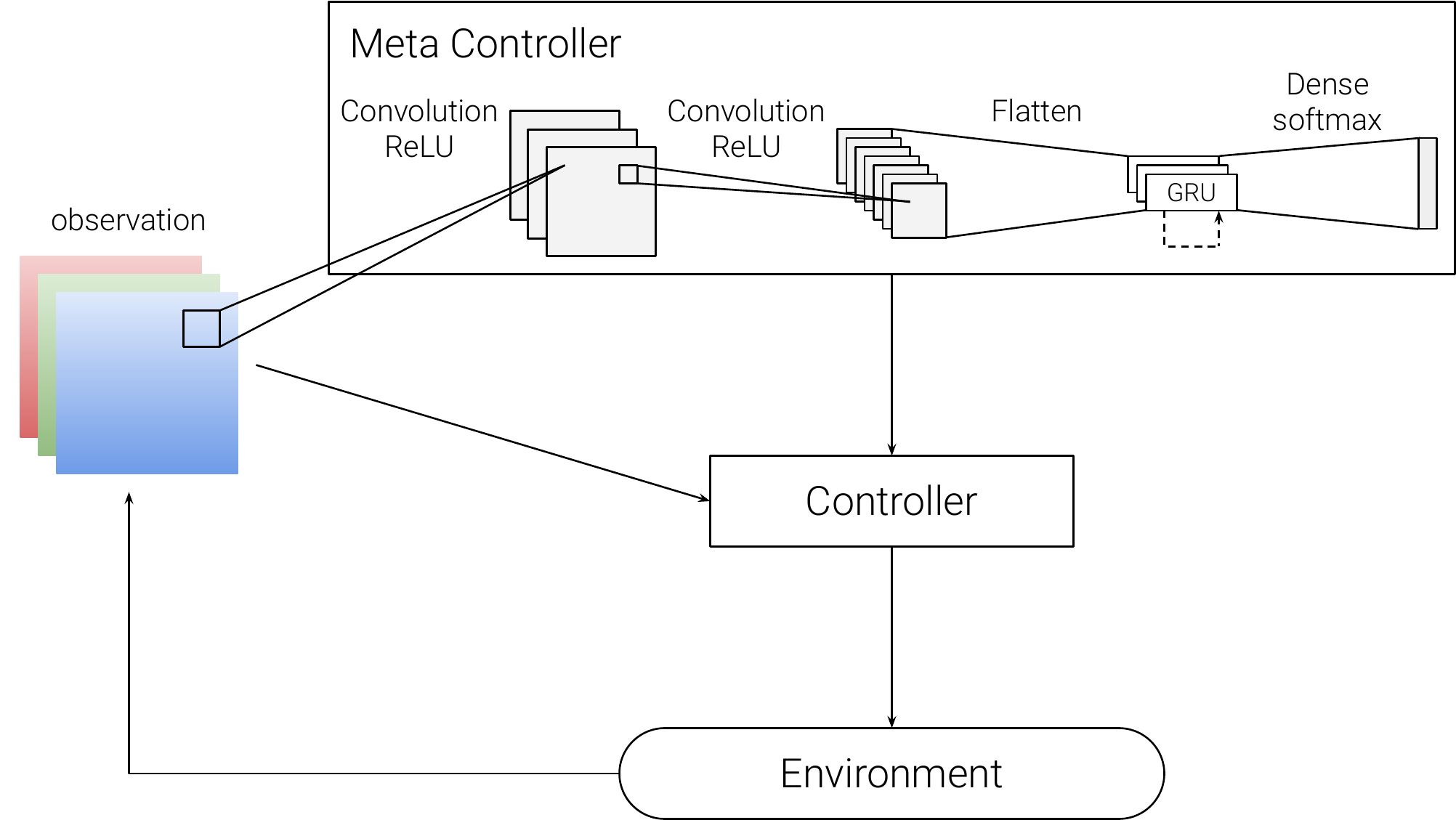}};
\draw node [right] at (1.2,0.3){\small$g\sim\pi(\cdot|\tilde{s};\theta)$};
\draw node [right] at (1.2,-1.3){\small$a\sim\pi_a(\cdot|s_a,g;\theta_a)$};
\end{tikzpicture}
\caption{The architecture of Rh-REINFORCE for high dimensional input vectors. When the input dimension is low, the meta controller does not use convolutional layers. The controller has a similar network architecture to the meta controller without a recurrent layer.}\label{fig:arc}
\end{figure*}

\paragraph{Learning.} (Algorithm~\ref{alg}) The meta controller uses REINFORCE \cite{williams1992simple} to learn a parameterized policy $\pi(g|\tilde{s};\theta)$. (In this context $\tilde{s}$ is actually $s\tilde{s}$ in the expressive analysis.) The performance measure is the expected return $\mathbb{E}[G_t]$. The parameters $\theta$ is optimized by gradient ascent on $\mathbb{E}[G_t]$ in the direction
\setcounter{equation}{0}
\begin{equation}\label{eq:theta}
G_t\nabla\ln{\pi(g_{(t)}|\tilde{s}_{(t)};\theta)}
\end{equation}

The controller uses a one-step actor-critic method  \cite{konda2000actor} to learn a policy $\pi_a(a|s,g;\theta_a)$, and a value function $v(s,g;\theta_v)$. Compared to REINFORCE, this method uses the one-step return and the value function as a baseline to reduce variance. The parameters $\theta_a$ is optimized using the gradient
\begin{equation}\label{eq:theta_a}
\delta \nabla\ln{\pi_a(a_{(t)}|s_{(t)},g;\theta_a)},
\end{equation}
where $\delta = i_t + \gamma v(s_{(t+1)},g;\theta_v) - v(s_{(t)},g;\theta_v)$ is the temporal difference error. ($i_t$ is an intrinsic reward.)

The loss for the value function is $L=\mathbb{E}[\delta^2]$. The parameters $\theta_v$ is optimized by gradient descent on $L$ in the direction
\begin{equation}\label{eq:theta_v}
\delta\nabla v(s_{(t)},g;\theta_v)
\end{equation}

In an episode, the meta controller selects a goal $g$ (line 7). Until $g$ is achieved or the terminal state is reached, the controller takes actions (line 11); after each action, $\theta_a$ and $\theta_v$ are updated using gradients (\ref{eq:theta_a}) and (\ref{eq:theta_v}) (line 15). After $g$ is achieved, $(s,g,R)$ is saved to construct a trajectory (line 17). The process repeats until the episode terminates. At the end of the episode, the trajectory is used to compute $G_t$ in gradient (\ref{eq:theta}); then $\theta$ is updated using backpropagation through time for the length of the episode (line 21).

\begin{algorithm}[tb]
\caption{Rh-REINFORCE learning}
\label{alg}
Initialize meta controller policy $\pi(g|\tilde{s};\theta)$\\
Initialize controller policy $\pi_a(a|s,g;\theta_a)$\\
Initialize controller value function $v(s,g;\theta_v)$
\begin{algorithmic}[1] 
\For{$episode=1$ to $n$}
    \State $trajectory\gets[~]$
    \State $\tilde{s}\gets[~]$
    \State Initialize the environment and get state $s$
    \State Append $s$ to $\tilde{s}$
    \While {$s$ is not terminal}
        \State {Select goal $g\sim\pi(\cdot|\tilde{s};\theta)$}
        \State $R\gets0$
        \State $s_a\gets s$
            \While {\textbf{not} ($s_a$ is terminal \textbf{or} $g$ reached)}
                \State Take action $a\sim\pi_a(\cdot|s_a,g;\theta_a)$
                \State Get $s'$, external reward $r$, intrinsic reward $i$
                \State $R\gets R+r$
                \State $s_a\gets s'$
                \State Update $\theta_a$, $\theta_v$
            \EndWhile
            \State Append $(s,g,R)$ to $trajectory$
            \State $s\gets s_a$
            \State Append $s$ to $\tilde{s}$
    \EndWhile
    \State Update $\theta$
\EndFor
\end{algorithmic}
\end{algorithm}

\section{Experiments}

To demonstrate that RHF is more expressive than HF, we compare Rh-REINFORCE with two HF systems: h-REINFORCE (substituting a fully-connected layer for the GRU layer in Rh-REINFORCE) and h-DQN. The systems are trained in three variants of the corridor environment and a grid environment. In each environment, the systems use an optimal controller, which takes the actions that result in the shortest path from the current state to a state that satisfies a given goal (or the most likely actions if the environment is stochastic).  Therefore, any difference in performance between the systems is due to their meta controllers.

\subsection{Environments}

\paragraph{Corridor.} (Figure~\ref{fig:corridor} (b)) The environment is corridor task (ii). It has 7 states from left to right: $s_0, s_1, \dots, s_6$. $s_0$ is the terminal state (gray). The actions are left move and right move. The agent starts in $s_3$ (circled); at each time step, it can move to the state immediate to its left or right. (Taking a right move in $s_6$ has no effect.) To constitute a visit to $s_6$ (asterisk), the agent must move from $s_5$ to $s_6$. The agent receives a reward of $+1$ if it visits $s_6$ at least twice and then reaches $s_0$; otherwise, the reward is $+0.01$. Actions have no rewards. An episode ends after 20 time steps; the agent receives a reward of $0$ if it fails to reach $s_0$.

\paragraph{Stochastic Corridor.} This a modified version of Corridor based on the environment in \cite{kulkarni2016hierarchical}. It has the same state configuration and reward function as Corridor. The difference it that when the agent takes a right move, with a probability of 0.5, it ends up in the state to its right, and with a probability of 0.5, it ends up in the state to its left. An episode ends after 20 time steps resulting in a reward of $0$ (the same as Corridor).

\paragraph{Doom Corridor.} This is Corridor set in a Doom map created using the ViZDoom API \cite{Kempka2016ViZDoom}. It has the same reward function as Corridor. The agent always faces right and observes $320\times240\times3$ RGB game frames. Each input frame is resized to $100\times100\times3$ before passed to a convolutional layer.

\paragraph{Grid.} (Figure~\ref{fig:grid}) This is a navigation task in a $5\times5$ gridworld. The terminal state (gray) is above the first row. The landmark states are numbered 1, 2, and 3. The agent starts at the top left state (circled). To obtain a reward of $+1$, the agent must visit landmark 1, return to the circled state, then visit landmark 2, return to the circled state, then visit landmark 3, return to the circled state, and finally reach the terminal state. The visits can be intertwined, for example, $0~1~0~1~2~1~0~3~0~\tau$ still results in a reward of $+1$. In all other cases the reward is $0$. An episode ends after 60 actions. The task cannot be solved by HF because it requires a state-goal trajectory that contains at least $s_0g_1$ and $s_0g_\tau$, which as per Theorem 1 cannot be generated by HF.

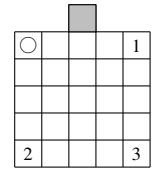
\begin{figure}
\centering
\begin{tikzpicture}[scale=0.36]
    \foreach \i in {0,...,5} {
        \draw  (\i,0) -- (\i,5);
    }
    \foreach \i in {0,...,5} {
        \draw  (0,\i) -- (5,\i);
    }
    \draw (2,5) -- (2,6);
    \draw (3,5) -- (3,6);
    \draw (2,6) -- (3,6);
    \draw node at (0.5,4.5) {\fontsize{9}{9}\Circle};
    \draw node at(4.5,4.5){\scriptsize1};
    \draw node at(0.5,0.5){\scriptsize2};
    \draw node at(4.5,0.5){\scriptsize3};
    \draw [fill=lightgray] (2, 5) rectangle (3, 6);
\end{tikzpicture}
\caption{The grid environment. The starting state is circled. The terminal state is gray. The landmarks are numbered 1, 2, and 3.}\label{fig:grid}
\end{figure}

\subsection{Training}
In each environment, all three systems use the same fixed controller. Hence only the meta controller's parameters are updated.

Rh-REINFORCE and h-REINFORCE use Algorithm~\ref{alg} to learn softmax policies. (Intrinsic rewards and line 15  are not applicable because only the meta controller is being trained.) In Corridor and Doom Corridor, the training includes a random exploration phase for the initial 1000 episodes, where the meta controller selects random goals. This phase is not used in the other two environments.

h-DQN uses a one-step Q-learning algorithm with experience replay: the replay size is 100000; the batch size is 64; the target network update rate is 0.001. The goal selection is $\epsilon$-greedy. $\epsilon$ decays from 1 to 0.01 over 15000 steps. Other hyperparameters are summarized in Table~\ref{tab}.

\begin{table*}[h!]
  \begin{center}
    \caption{Meta controller hyperparameters. All three systems use two convolution layers followed by a GRU/Dense layer in Doom Corridor. Rh-REINFORCE and h-REINFORCE use the same output activation, loss, and optimizer.  h-REINFORCE and h-DQN have the same network architectures except the output activation, loss, and optimizer. The same learning rate is used across the systems.}
    \label{tab}
    \begin{tabular}{cc|c|c}
     \toprule 
    & Rh-REINFORCE & h-REINFORCE  & h-DQN \\
      \midrule
  Corridor/Stochastic  & \multirow{2}{*}{GRU: 64 units}  & \multicolumn{2}{c}{Dense 1: 16 units [ReLU]} \\ 
  Grid  &   & \multicolumn{2}{c}{{Dense 2: 32 units [ReLU]}}\\
     \midrule
     \multirow{3}{*}{Doom Corridor} & \multicolumn{3}{c}{Conv 1: 32 filters, (8, 8) kernel, strides=4 [ReLU]}  \\ 
     & \multicolumn{3}{c}{Conv 2: 64 filters, (4, 4) kernel, strides=2 [ReLU]}   \\ 
    & GRU: 256 units &\multicolumn{2}{c}{Dense: 256 units [ReLU]} \\
         \midrule
         Output activation   & \multicolumn{2}{c|}{softmax} & linear\\
        Loss   & \multicolumn{2}{c|}{categorical crossentropy} & Huber\\
        Optimizer   & \multicolumn{2}{c|}{Adam} & RMSprop\\
        Learning rate   & \multicolumn{3}{c}{0.001}\\
      \bottomrule
    \end{tabular}
  \end{center}
\end{table*}

\subsection{Results}

Rh-REINFORCE outperforms h-REINFORCE and h-DQN in all the environments. Figure~\ref{fig:exp} shows the average episodic returns of 10 runs in each environment.

In Corridor, Rh-REINFORCE learns an optimal policy in 2000 episodes. h-REINFORCE and h-DQN achieve an average return of 0.013 and 0.116 respectively after 10000 episodes.

In Stochastic Corridor, Rh-REINFORCE, h-REINFORCE, and h-DQN achieve an average return of 0.416, 0.071, and 0.054, respectively. For comparison, a handcrafted optimal policy has an average return of 0.423; a random policy has an average return of 0.031.

In Doom Corridor, Rh-REINFORCE, h-REINFORCE, and h-DQN achieve an average return of 0.835, 0.029, and 0.006, respectively.

In Grid, Rh-REINFORCE learns an optimal policy in 14000 episodes. h-DQN and h-REINFORCE are unable find an optimal policy after 20000 episodes.

\begin{figure*}[!htbp]
    \centering
    \begin{subfigure}{0.33\textwidth}
        \centering
        \includegraphics[width=\textwidth]{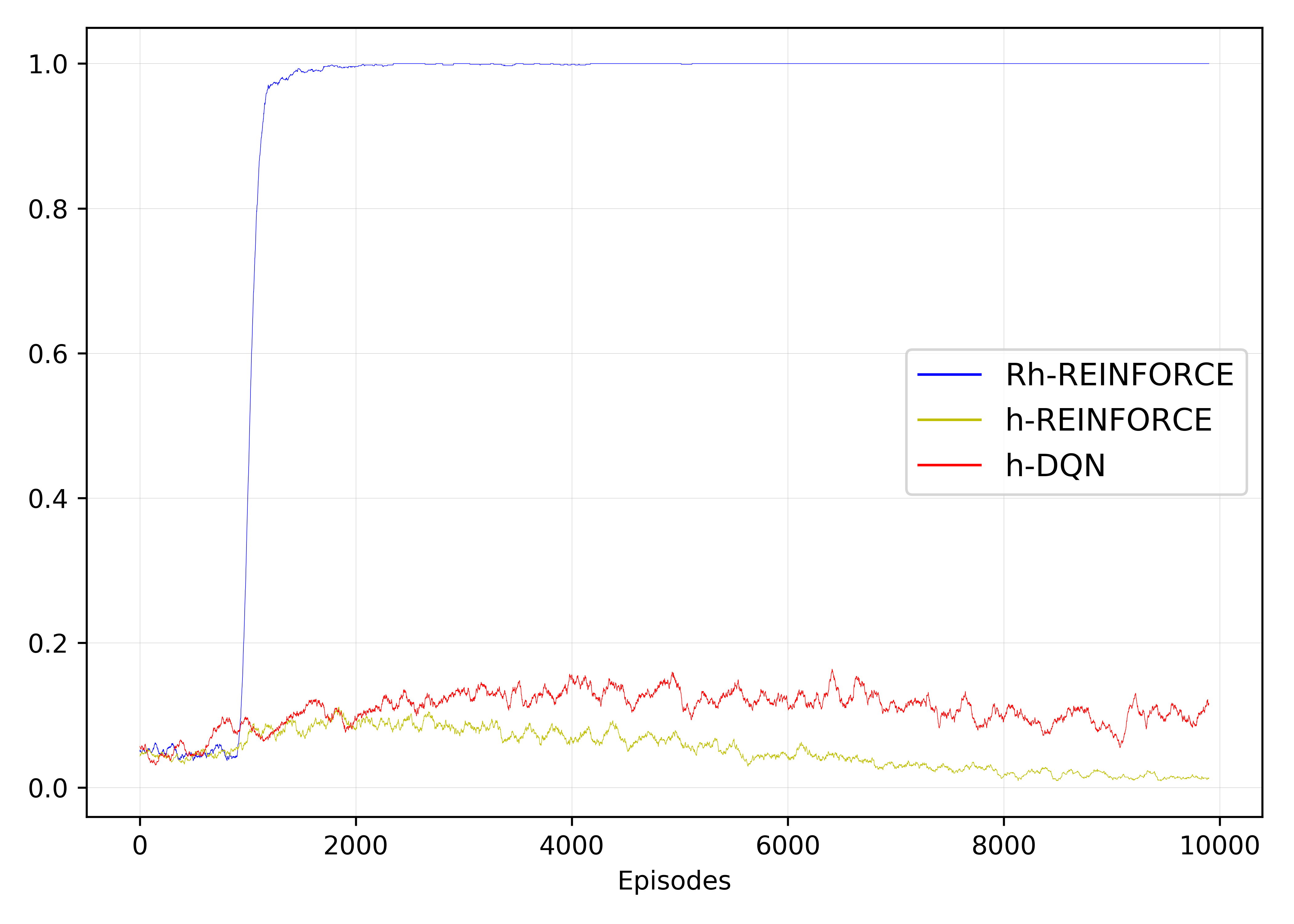}
        \caption{Corridor}
    \end{subfigure}
    \hspace{2em}
    \begin{subfigure}{0.33\textwidth}
        \centering
        \includegraphics[width=\textwidth]{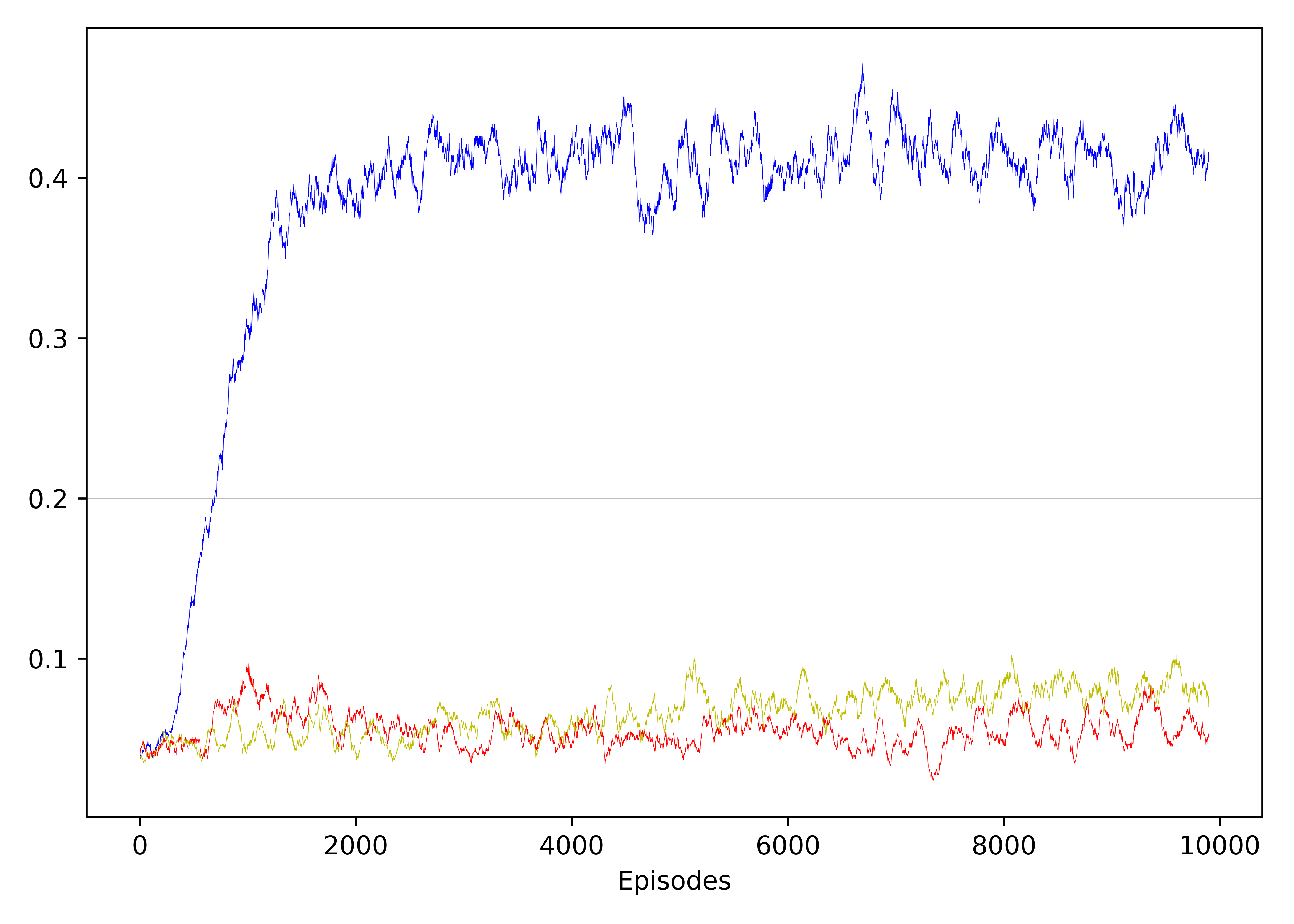}
        \caption{Stochastic Corridor}
    \end{subfigure}
    \\[1em]
    \begin{subfigure}{0.33\textwidth}
        \centering
        \includegraphics[width=\textwidth]{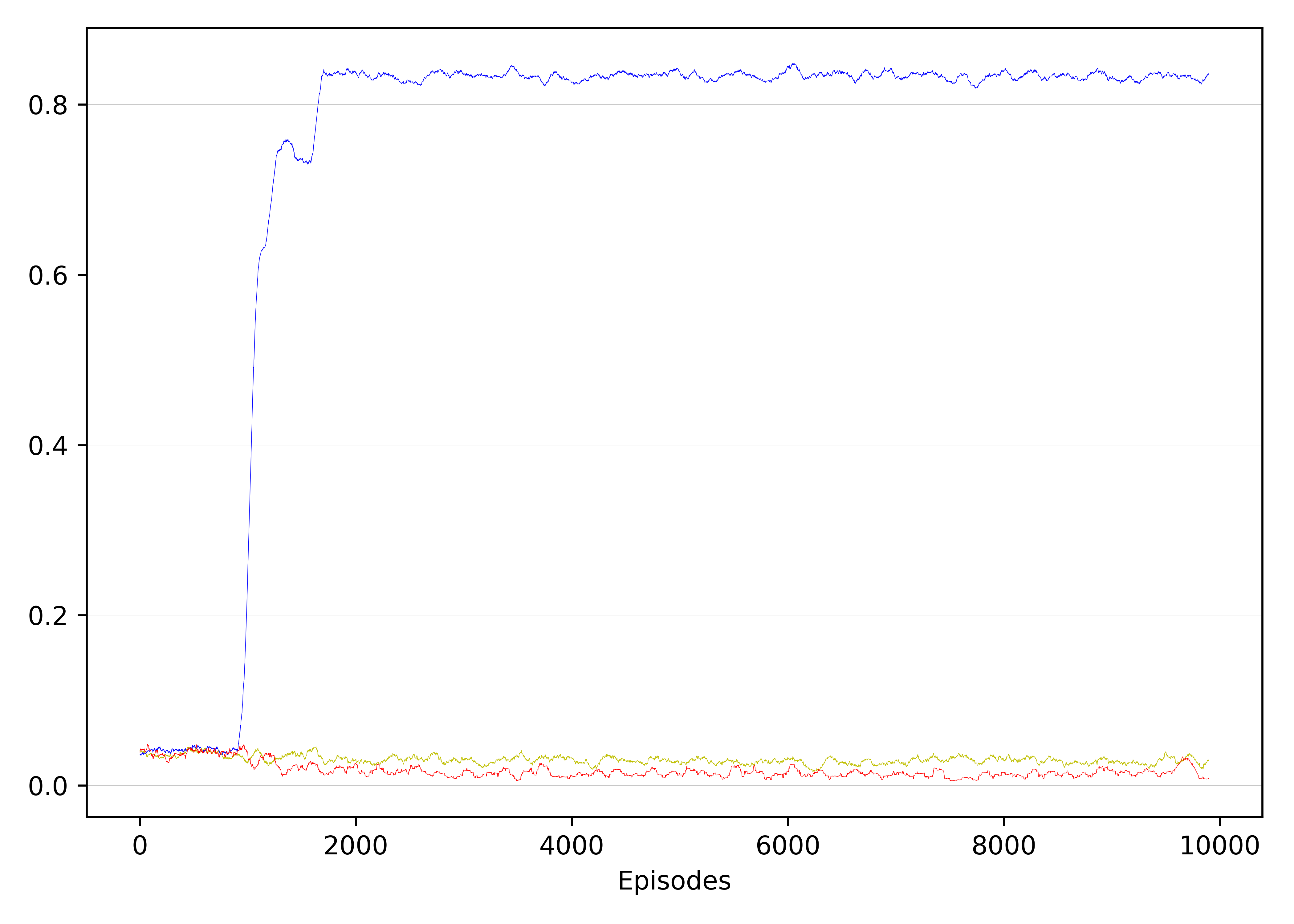}
        \caption{Doom Corridor}
    \end{subfigure}
    \hspace{2em}
    \begin{subfigure}{0.33\textwidth}
        \centering
        \includegraphics[width=\textwidth]{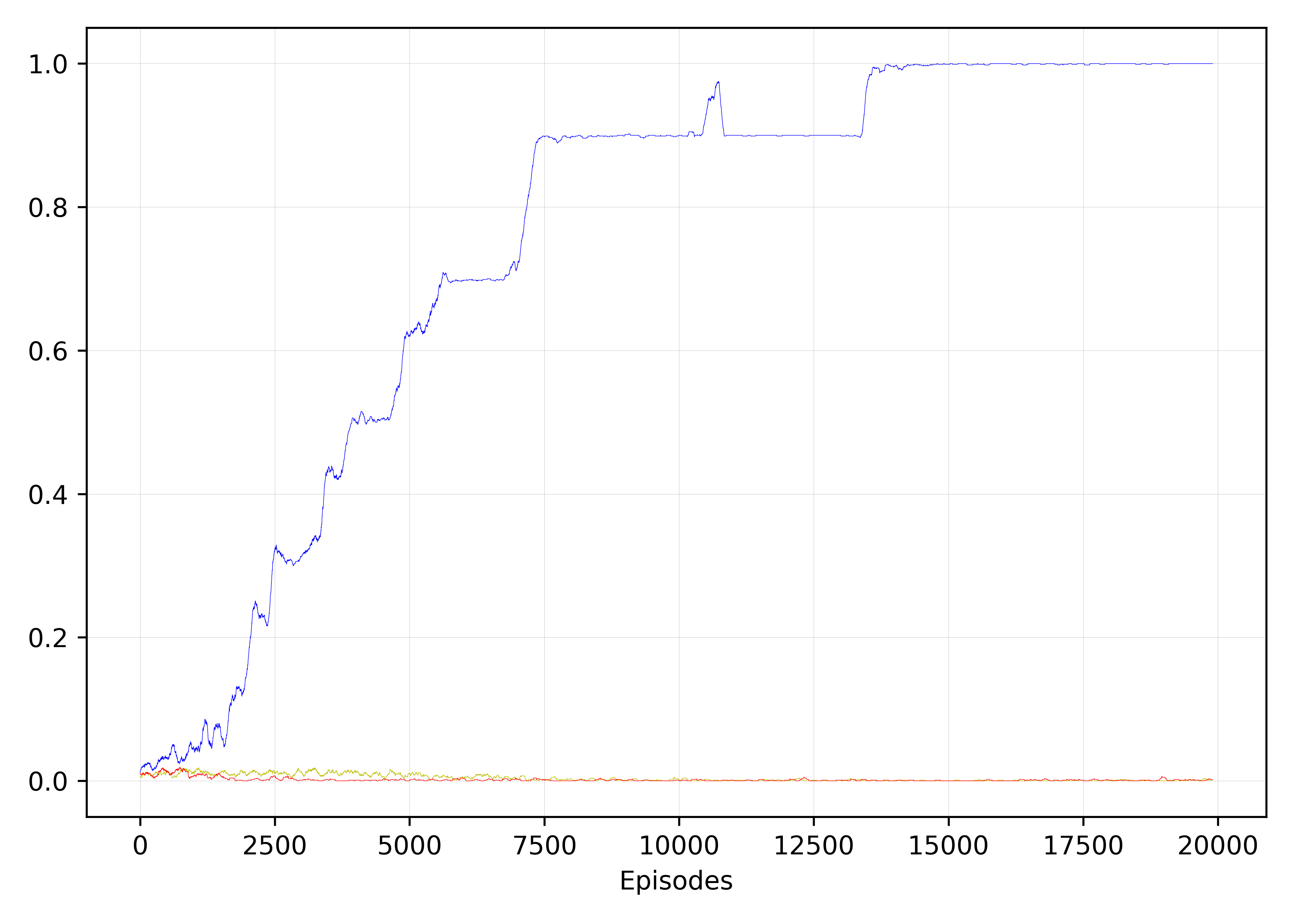}
        \caption{Grid}
    \end{subfigure}
    \caption{Comparison of Rh-REINFORCE (blue), h-REINFORCE (yellow), and h-DQN (red). Episodic returns are plotted against training episodes. The data points are the average of 10 runs smoothed by a 100-episode moving average. In each environment, all three systems use the same controller. Thus the plots illustrate the difference in meta controller performance.}
    \label{fig:exp}
\end{figure*}

\section{Concluding Remarks}
The experimental results are consistent with our expressiveness analysis. In the deterministic environments, the two HF systems are unable to generate the state-goal trajectories that result in the maximum reward; in contrast, the RHF system is capable of learning to generate the trajectories and thus obtaining a much higher reward. Since the policies (i.e., softmax and $\epsilon$-greedy) during training are not deterministic, it is possible for the HF systems to obtain a high reward by chance, but they are unable to do so consistently. The same conclusion holds in the stochastic environment.

We claim neither that an HF or RHF system can optimally solve an arbitrary RL task nor that they perform better than other deep RL systems. It is possible for an RL system without a hierarchical architecture to achieve a high performance in a variety of tasks. Our analysis focuses on the expressiveness of HF and RHF, i.e., the kinds of state-goal trajectories that the frameworks can and cannot generate.

For future work, we want to investigate the expressiveness of architectures that have additional recurrent levels on top of the meta level. We conjecture that a single recurrent level with a sufficient number of units can simulate any number of recurrent levels and thus adding more levels will not increase a system's expressive power. 

\section*{Acknowledgements} This research was supported by the Office of Naval Research grants N00014-18-1-2009 and  N68335-18-C-4027 and the National Science Foundation grant 1909879. The views, opinions, and findings expressed are those of the authors and should not be interpreted as representing the official views or policies of the Department of Defense or the U.S. Government.

\bibliographystyle{kr}
\bibliography{kr}

\end{document}